\documentclass{article}
\usepackage{amsmath,amsfonts,amssymb,amsthm,graphicx,etoolbox,epigraph}

\usepackage{algorithm}
\usepackage[noend]{algpseudocode}
\usepackage[colorlinks=true,citecolor=blue,pdfpagemode=UseNone,pdfstartview=FitH]{hyperref}

\allowdisplaybreaks[4]

\newcommand{\I}{^{\textnormal{\tiny I}}}
\newcommand{\II}{^{\textnormal{\tiny II}}}

\newcommand{\dd}{\mathrm{d}}
           
\newcommand{\R}{\mathbb{R}}

\newcommand{\UP}{\overline{\mathbb{P}}}

\newcommand{\K}{\mathcal{K}}
\newcommand{\Loss}{\mathrm{Loss}}

\newcommand{\FFF}{\mathcal{F}}

\theoremstyle{plain}
\newtheorem{theorem}{Theorem}[section]

\newtheorem{proposition}[theorem]{Proposition}

\theoremstyle{definition}
\newtheorem{protocol}[theorem]{Protocol}

\theoremstyle{remark}
\newtheorem{remark}[theorem]{Remark}

\newlength{\IndentI}
\newlength{\IndentII}
\newlength{\IndentIII}
\newlength{\IndentIV}
\newlength{\WidthI}
\newlength{\WidthII}
\newlength{\WidthIII}
\newlength{\WidthIV}
\newcommand{\indentI}{\noindent\hspace*{\IndentI}}
\newcommand{\indentII}{\noindent\hspace*{\IndentII}}

\newcommand{\indentIV}{\noindent\hspace*{\IndentIV}}

\begin{document}
\title{Logic of subjective probability}
\author{Vladimir Vovk}
\maketitle
\begin{abstract}
  In this paper I discuss both syntax and semantics of subjective probability.
  The semantics determines ways of testing probability statements.
  Among important varieties of subjective probabilities
  are intersubjective probabilities and impersonal probabilities,
  and I will argue that well-tested impersonal probabilities acquire features of objective probabilities.
  Jeffreys's law, my next topic, states that two successful probability forecasters
  must issue forecasts that are close to each other,
  thus supporting the idea of objective probabilities.
  Finally, I will discuss connections between subjective and frequentist probability.

  The version of this paper at \url{http://vovk.net/lop/}
  is updated most often.
\end{abstract}

\setlength{\epigraphwidth}{0.5\textwidth} 
\epigraph{It may well be the fate of any theory we can ever devise to be eventually falsified;
  we can only hope that our currently unfalsified theory or theories will prove useful
  for understanding and predicting the world, during their limited lifespan.}%
  {A. Philip Dawid \cite[Sect.~2.2]{Dawid:2004}}

\section{Introduction}

The logic of probability,
as understood in this paper,
consists of two main parts, syntax and semantics.
(This notion of logic agrees with the standard understanding in mathematical logic
\cite{Shoenfield:1967}
and was used in my earlier paper \cite{Vovk:1993}.)
The syntax of probability is its mathematical language,
such as Kolmogorov's axioms \cite{Kolmogorov:1933}.
Mathematicians have been very successful in developing
probability theory as a mathematical theory
without worrying about its connections with reality.

The semantics of probability is about the meaning of probability statements.
When are they true and when are they false?
In more pragmatic terms,
how do we verify and falsify probability statements?
Falsification is often easier than verification,
especially for probability statements.
Popper \cite{Popper:1959} emphasized the importance of falsification
and difficulties of verification even for scientific statements
that do not involve probabilities.
For probability statements,
it is commonplace in statistical hypothesis testing
that while we are often justified in rejecting probability statements
(the probability of being wrong is under our control),
in non-trivial cases we are not justified in accepting them.
Therefore, my emphasis in discussing semantics will be on falsification.

There will be little originality in this paper.
It describes my understanding of ideas developed
in collaboration with Glenn Shafer
and Philip Dawid (Philip's prequential principle inspired all these developments).
For fuller expositions,
see the books \cite{Shafer/Vovk:2001} and \cite{Shafer/Vovk:2019},
Dawid's position statement \cite{Dawid:2004},
and Shafer's discussion paper \cite{Shafer:2021}.

I start in Sect.~\ref{sec:syntax} from the syntax of probability.
Two natural ways to define the syntax of probability
are the standard measure-theoretic (Kolmogorov's axioms \cite{Kolmogorov:1933})
and an equally old but less popular game-theoretic \cite{Shafer/Vovk:2019}.
In this paper I will concentrate on the game-theoretic picture,
which explicitly involves time from the very beginning.
The game-theoretic picture is in some respects simpler than measure-theoretic.
We typically consider protocols with three players.
Reality is the basic player producing a stream of observations.
Forecaster gives predictions for future observations;
it represents the syntax.
The semantics is represented by Sceptic,
whose role is to test Forecaster's predictions.
Semantics is discussed in Sect.~\ref{sec:semantics}
and is based on what we call the martingale interpretation of probability.

In Sect.~\ref{sec:objective} I will introduce intersubjective,
impersonal, and objective probabilities,
illustrating them with several examples.
Jeffreys's law and frequentist probability
will be treated in Sects~\ref{sec:Jeffreys} and~\ref{sec:frequentist},
respectively.

This paper has been prepared to accompany my talk at MSP 2023,
the third workshop on ``The Mathematics of Subjective Probability''
(University of Milano-Bicocca, 11 September 2023).

\section{Syntax of probability}
\label{sec:syntax}

In this paper we will be interested in repeated prediction,
which covers a wide range of applications in science and statistics,
as will be discussed later.
Our task is to predict observations $y_1,y_2,\dots$,
which are elements of a measurable space $\mathbf{Y}$,
usually assumed finite (with discrete $\sigma$-algebra) in this paper;
the predictions are often probability measures $P_1,P_2,\dots$ on $\mathbf{Y}$.

The most standard syntax is provided by measure-theoretic probability \cite{Kolmogorov:1933}
(which does not even require repeated prediction,
although we are only interested in this case).
Under Kolmogorov's axioms,
there is an overarching probability measure $Q$ on $\mathbf{Y}^{\infty}$,
and $P_n$ is computed from $Q$ as the conditional probability for $y_n$
given the past (including $y_1,\dots,y_{n-1}$).

A less standard syntax is given by the following \emph{one-step-ahead prediction protocol}
(in which $\mathfrak{P}(\mathbf{Y})$ stands for the set of all probability measures on $\mathbf{Y}$).
\begin{protocol}\label{prot:basic}\ \\
  \indentI FOR $n=1,2,\dots$:\\
    \indentII Forecaster announces $P_n\in\mathfrak{P}(\mathbf{Y})$\\
    \indentII Reality announces the actual observation $y_n\in\mathbf{Y}$.
\end{protocol}

\noindent
One-step-ahead prediction is standard in game-theoretic probability
\cite[Sect.~7]{Vovk/Shafer:2023}
and was the only kind of prediction considered
in \cite{Vovk:1993,Dawid/Vovk:1999,Shafer/Vovk:2001,Shafer/Vovk:2019}.
One-step-ahead prediction covers important applications
(see, e.g., the examples in Sect.~\ref{sec:objective}),
but in this paper we will also discuss more general cases
(as in \cite{Vovk:2023Bayesian}).

The most basic such case is two-steps-ahead point prediction,
described by the following protocol:
\begin{protocol}\label{prot:2-steps-point}\ \\
  \indentI FOR $n=1,2,\dots$:\\
    \indentII Forecaster announces $a_n,b_n\in\R$\\
    \indentII Reality announces the actual observation $y_n\in\R$.
\end{protocol}

\noindent
Intuitively, $a_n$ is the prediction for the next step
and $b_n$ for the step after that.
This is what we typically see when we use a weather app
(where the maximal temperature $y_n$ can be predicted for, say, 7 days ahead).

We can have similar protocols for probabilistic predictions
two (more generally, $K$) steps ahead.
There are two varieties of such protocols,
for marginal and for joint probabilities.
(The latter are treated in detail in \cite{Vovk:2023Bayesian}.)
In the case of marginal probabilistic prediction,
the protocol is (in the two-steps-ahead case):

\begin{protocol}\label{prot:2-steps-margin}\ \\
\indentI FOR $n=1,2,\dots$:\\
  \indentII Forecaster announces $P_n,Q_n\in\mathfrak{P}(\mathbf{Y})$\\
  \indentII Reality announces the actual observation $y_n\in\mathbf{Y}$.
\end{protocol}

\noindent
Now $P_n$ is the prediction for the next trial
and $Q_n$ for the trial after that.
Let us assume that $\mathbf{Y}$ is a finite set.

Protocol~\ref{prot:2-steps-margin} is a natural modification of Protocol~\ref{prot:2-steps-point},
but if $\mathbf{Y}$ is a small set (e.g., of cardinality 2),
we might want to have a joint prediction for $(y_{n+1},y_{n+2})$
rather than separate predictions for $y_{n+1}$ and $y_{n+2})$.
This leads to the following protocol:

\begin{protocol}\ \\
\indentI FOR $n=1,2,\dots$:\\
  \indentII Forecaster announces $P_n\in\mathfrak{P}(\mathbf{Y}^2)$\\
  \indentII Reality announces $y_n\in\mathbf{Y}$.
\end{protocol}

\noindent
Now $P_n$ is the joint prediction for the next two trials.

Let me give an example of a more complex protocol
that involves, additionally, an element of decision making.
(After all, the focus of the day of my talk, 11 September 2023,
is supposed to have decision theory as its focus.)
Let $K\in\{1,2,\dots\}$;
the following protocol describes predicting $K$ steps ahead.

\begin{protocol}\label{prot:DM}\ \\
  \indentI FOR $n=1,2,\dots$:\\
    \indentII Reality announces $\lambda_n:\mathbf{D}\times\mathbf{Y}^{K}\to[0,1]$\\
    \indentII Forecaster announces $P_n\in\mathfrak{P}(\mathbf{Y}^{K})$\\
    \indentII Decision Maker announces $d_n\in\mathbf{D}$\\
    \indentII Reality announces $y_n\in\mathbf{Y}$.
\end{protocol}

\noindent
There is a new player, Decision Maker,
who chooses moves (her decisions) from a set $\mathbf{D}$ of available decisions;
we will assume that the set $\mathbf{D}$ is finite.
At the beginning of each step $n$ Reality announces a loss function $\lambda_n$
determining Decision Maker's loss
$\lambda_n(d_n,y_n\dots y_{n+K-1})$,
which depends on the observations $y_i$ output by Reality over the next $K$ steps.
We assume that the loss function is bounded and normalized to take values in the interval $[0,1]$.

A natural strategy for Decision Maker in this protocol is the \emph{Bayes optimal strategy}
\begin{equation}\label{eq:B}
  d_n
  \in
  \arg\min_{d\in\mathbf{D}}
  \sum_{x\in\mathbf{Y}^{K}}
  \lambda_n(d,x) P_n(\{x\})
\end{equation}
(to break ties, if any,
let us set $d_n$ to the smallest element of the $\arg\min$
in a fixed linear order on $\mathbf{D}$).
However, without some semantics of probability,
we cannot say anything about the quality of this strategy.

\begin{remark}
  It is not immediately clear from the formal description
  of Protocol~\ref{prot:DM} that, intuitively,
  we do not assume that $y_n$ is the only information that Forecaster gets on step $n$.
  If $y_n$ is the only information and $K=\infty$,
  it is usually assumed that the forecasts $P_n$ are updated by conditioning,
  \begin{equation}\label{eq:updating}
    P_{n+1}(A)
    :=
    \frac{P_{n}(y_{n}A)}{P_{n}(y_{n}\mathbf{Y}^{\infty})}
  \end{equation}
  for each measurable $A\subseteq\mathbf{Y}^{\infty}$
  (remember that $\mathbf{Y}$ is finite).
  The principle that the new observations $y_n$
  are the only new information is known as the \emph{principle of total evidence} \cite{Shafer:1985}.
  Lewis \cite{Lewis:1999} derives updating via conditioning
  using his diachronic Dutch Book argument,
  which implicitly uses the principle of total evidence.
  If $K<\infty$ (which is the case we are interested in),
  we have to modify \eqref{eq:updating},
  but the marginal $P_{n+1}$-probability for the first $K-1$ observations
  is still obtained by conditioning $P_n$ on observing $y_n$.
  However, there are many cases where the principle of total evidence does not hold.
  Consider, e.g., predicting rain $K$ days ahead by a national weather forecaster,
  such as Servizio Meteorologico.
  Even if we are recording only the fact of rain for each day,
  $\mathbf{Y}=\{\text{rain},\text{dry}\}$,
  Forecaster obtains plenty of other data from almost 200 weather stations.
\end{remark}

\section{Criterion of empirical adequacy (semantics)}
\label{sec:semantics}

In principle,
it is up to the probability forecaster to explain the meaning of his or her forecasts.
Ideally, the forecaster should tell us under what circumstances
they will accept that their forecasts have been poor.
However, at this time forecasts issued by people and organizations
(such as weather forecasts)
are rarely (or never) accompanied by such explanations.
The \emph{martingale interpretation of probability},
to be explained shortly, has been proposed
as a default understanding of the meaning of probability forecasts.

The martingale interpretation of probability starts from a way of testing probability forecasts,
and this then determines their empirical content.
In order to test Forecaster,
we introduce another player, Sceptic,
who tests Forecaster in a very traditional way
\cite{Ramsey:1926,deFinetti:1937},
by betting against him.
For example, Protocol~\ref{prot:basic} extends to the following testing protocol.

\begin{protocol}\label{prot:testing}\ \\
  \indentI $\K_0 := 1$\\
  \indentI FOR $n=1,2,\dots$:\\
    \indentII Forecaster announces $P_n\in\mathfrak{P}(\mathbf{Y})$\\
    \indentII Sceptic announces measurable $f_n:\mathbf{Y}\to[0,\infty]$
        such that $P_n(f_n) = 1$\\
    \indentII Reality announces the observation $y_n\in\mathbf{Y}$\\
    \indentII $\K_n := \K_{n-1} f_n (y_n)$.
\end{protocol}

I will refer to $\K_n$ as Sceptic's capital at time $n$,
but this is an imaginary gambling picture.
Sceptic earns ``paper money'',
and the real meaning of $\K_n$ is the degree to which the forecasts have been falsified.
Sceptic start from a unit capital, $\K_0 = 1$,
and we do not allow his capital to become negative
(as soon as it does, the game is stopped and the attempt at discrediting Forecaster fails).
At each step $n$ Sceptic gambles using his available capital,
and the gamble is fair in that the expected relative increase in his capital is 1:
$P_n(f_n) = 1$, where $P_n(f_n)$ is de Finetti's notation for $\int f_n \dd P_n$.

What values of $\K_n$ are regarded as discrediting Forecaster?
A possible scale is the one suggested
by Harold Jeffreys \cite[Appendix B]{Jeffreys:1961} for Bayes factors:
\begin{itemize}
\item
  If $\K_n<1$, there is no evidence against Forecaster.
\item
  If $\K_n\in(1,\sqrt{10})\approx(1,3.16)$,
  the evidence against Forecaster is ``not worth more than a bare mention''.
\item
  If $\K_n\in(\sqrt{10},10)\approx(3.16,10)$,
  the evidence against Forecaster is \emph{substantial}.
\item
  If $\K_n\in(10,10^{3/2})\approx(10,31.6)$,
  the evidence against Forecaster is \emph{strong}.
\item
  If $\K_n\in(10^{3/2},100)\approx(31.6,100)$,
  the evidence against Forecaster is \emph{very strong}.
\item
  If $\K_n>100$,
  the evidence against Forecaster is \emph{decisive}.
\end{itemize}

If Sceptic follows some strategy
(a function of Reality's and Forecaster's previous moves),
Sceptic's capital becomes a function of Reality's and Forecaster's moves,
and such functions are referred to as \emph{game-theoretic martingales}
(also known as \emph{farthingales} \cite[Sect.~3.9]{Dawid/Vovk:1999}).
This is the origin of the name ``martingale interpretation of probability'',
but testing can also happen on the fly, without Sceptic following a strategy.

\begin{remark}
  In applications,
  for the testing procedure to be valid
  it is important that the observations should be recorded fully according to a pre-specified protocol,
  without the kinds of cheating discussed by de Finetti in \cite[Sect.~1.4]{deFinetti:2017}.
\end{remark}

\begin{remark}
  Bruno de Finetti widely used the principle of \emph{coherence},
  one version of which is not being susceptible to a Dutch Book;
  see, e.g., \cite{deFinetti:1937} and \cite[Appendix, Sects~16 and~17]{deFinetti:2017}.
  It is a condition of consistency of forecasts made.
  Coherence is an important ingredient of the martingale interpretation of probability
  (and is discussed in detail in \cite{Shafer/Vovk:2019}),
  but per se it does not establish any connections between forecasts and reality.
\end{remark}

To connect the game-theoretic picture of Protocol~\ref{prot:testing}
with the standard stochastic picture,
let us consider the case where there is a ``true probability measure'' $Q$
in the background,
Reality generates her moves from $Q$,
and Forecaster computes his predictions also from $Q$.
(As in Dawid's super-strong prequential principle \cite[Sect.~5.2]{Dawid/Vovk:1999}.)
Formally, $Q$ is a probability measure on a measurable space $(\Omega,\FFF)$
equipped with a filtration $\FFF_0\subseteq\FFF_1\subseteq\dots\subseteq\FFF$,
$y_n$ is the outcome of an $\FFF_n$-measurable random variable $Y_n$,
and $P_n$ is a regular conditional distribution under $Q$ of $Y_n$ given $\FFF_{n-1}$
(assumed to exist).
Then, for a measurable strategy for Sceptic,
his capital will form a measure-theoretic test martingale,
where the adjective ``test'' refers to the initial value of the martingale being 1
and the martingale taking only nonnegative values.
To translate Jeffreys's thresholds into the usual thresholds
used in classical hypothesis testing,
we can use Ville's inequality \cite[p.~100]{Ville:1939}:
for any $c\ge1$,
\[
  Q(\sup_n \K_n\ge c)
  \le
  1/c
\]
For example,
the probability that Sceptic ever finds decisive evidence against Forecaster
is at most 1\%.
Therefore, we can regard the martingale interpretation of probability
as a betting version of Cournot's principle \cite{Shafer:2007},
saying that prespecified events of low probability are not expected to happen.

\subsection{Other interpretations of probability}
\label{subsec:other}

There is another version of our criterion of Forecaster's empirical adequacy,
the infinitary one:
we regard Forecaster to be discredited at time $\infty$
when $\K_n\to\infty$ as $n\to\infty$.
It was used in \cite{Dawid/Vovk:1999} and discussed in \cite[Sect.~1.3]{Shafer/Vovk:2001}.
It is clear that it is not applicable in practice.

\begin{remark}
  The martingale interpretation of probability can be regarded
  as a betting variety of Popper's propensity interpretation of probability,
  even though the terminology used in this paper is very different from Popper's.
  Popper's interpretation, however, has an important difference:
  it relies on what he calls ``a quasi-Cournot principle''
  \cite[Sect.~24, point (c)]{Popper:1983},
  which differs from Cournot's principle by replacing probabilities close to 0
  by probabilities exactly equal to 0.
  Probabilities exactly equal to 0 were also considered by Dawid
  \cite[Sect.~3.3.4]{Dawid:2004}.
\end{remark}

\begin{remark}\label{rem:psychology}
  The most popular semantics for subjective probabilities is a psychological one:
  e.g., a probability distribution may correspond to the state of the mind of a subject,
  and a putative correspondence may be tested by asking the subject
  to accept or reject various bets.
  We are not discussing it in this paper,
  but the literature on this topic is extensive,
  starting from Ramsey \cite{Ramsey:1926} and de Finetti (see, e.g., \cite{deFinetti:1937,deFinetti:2017}).
\end{remark}

\begin{remark}\label{rem:logical}
  The logical interpretation of probability
  (see, e.g., Keynes \cite{Keynes:1921} and Carnap \cite{Carnap:1950})
  is in one respect similar to the psychological interpretation
  discussed in the previous remark.
  Keynes and Carnap argued that there are objective ways to attach probabilities
  to logical statements
  (given other statements, those describing evidence).
  Neither psychological nor logical interpretation
  attempt to connect probabilities with the outcomes of the events they describe.
  Nowadays the logical interpretation
  appears to be less popular than its main competitors.
  The objective Bayesian view
  (such as Jeffreys's \cite{Jeffreys:1961}; see also \cite{Berger:2006})
  can be regarded as a special case of the logical interpretation.
\end{remark}

The frequentist interpretation will be discussed in Sect.~\ref{sec:frequentist}.

\subsection{Two-steps-ahead prediction}
\label{subsec:2-steps}

One-step-ahead prediction is fundamental:
e.g., one step ahead is implicit in the notion of a martingale in probability theory
\cite[Chap.~7]{Shiryaev:2019}.
It's been the focus for game-theoretic probability so far
\cite[Sect.~7]{Vovk/Shafer:2023}.

Two-steps-ahead prediction is more complicated,
and to have a clear understanding of the process of prediction
a financial picture is very helpful.
We complement the basic forecasting protocol
with a ``market'' allowing the third player, Sceptic,
to trade in futures contracts
(these are the most standard financial derivatives;
see, e.g., \cite[Chap.~2]{Hull:2021} and \cite{Duffie:1989}).
Futures contracts is an old idea (see, e.g., \cite{Schaede:1989})
that arose gradually in financial industry,
but in our prediction protocols it will be a powerful way
of reducing two-steps-ahead prediction
(more generally, prediction multiple steps ahead)
to one-step-ahead prediction.

An extension of Protocol~\ref{prot:2-steps-point} including testing is:

\begin{protocol}\label{prot:2-steps-point-test}\ \\
  \indentI $\K_0 := 1$\\
  \indentI Forecaster announces $a_1,b_1\in\R$\\
  \indentI Sceptic announces $A_1,B_1\in\R$\\
  \indentI Reality announces $y_1\in\R$\\
  \indentI $\K'_1 := \K_0 + A_1(y_1-a_1)$\\
  \indentI FOR $n=2,3,\dots$:\\
    \indentII Forecaster announces $a_n,b_n\in\R$\\
    \indentII $\K_{n-1} := \K'_{n-1} + B_{n-1}(a_n-b_{n-1})$
        \hfill\refstepcounter{equation}(\theequation)\label{eq:point-K-P}\\
    \indentII Sceptic announces $A_n,B_n\in\R$\\
    \indentII Reality announces $y_n\in\R$\\
    \indentII $\K'_n := \K_{n-1} + A_n(y_n-a_n)$.
        \hfill\refstepcounter{equation}(\theequation)\label{eq:point-K-y}
\end{protocol}

\noindent
At step $n$ Forecaster announces the prices $a_n$ and $b_n$
for two futures contracts, $\Phi_n$ and $\Phi_{n+1}$, respectively.
The futures contract $\Phi_n$, $n=1,2,\dots$, pays $F_n^+:=y_n$ at the end of step $n$.
Its price is first announced at step $n-1$ (unless $n=1$) as $F_n^-:=b_{n-1}$
and then revised at step $n$ to $F_n:=a_n$.
At step $n$ Sceptic takes positions $A_n$ and $B_n$ in the two futures contracts,
$\Phi_n$ and $\Phi_{n+1}$ respectively,
that are traded at step $n$.
The change $\K_n-\K_{n-1}$ in Sceptic's capital between steps $n$ and $n+1$
has two components:
in \eqref{eq:point-K-y} we settle the futures contract $\Phi_n$
getting
\[
  F_n^+ - F_n
  =
  y_n - a_n
\]
per contract (this is final settlement),
and in \eqref{eq:point-K-P} of the following step
we settle the futures contract $\Phi_n$
(a different one since $n$ is now larger by 1)
getting
\[
  F_n - F_n^-
  =
  a_n - b_{n-1}
\]
per contract (this is intermediate settlement),
which agrees with Protocol~\ref{prot:2-steps-point-test}.

Notice that at the $n$th iteration of the FOR loop
in Protocol~\ref{prot:2-steps-point-test}
both increases in capital come from the futures contract $\Phi_n$:
first intermediate in \eqref{eq:point-K-P}
and then final in \eqref{eq:point-K-y}.
The initial price $F_n^-$ of $\Phi_n$ is announced at the previous iteration.

Protocols~\ref{prot:2-steps-margin}--\ref{prot:DM} can be extended
in similar ways to allow testing.
For example, in the case of Protocol~\ref{prot:DM}
we allow Sceptic to trade in future contracts $\Phi(x)$,
where $x\in\mathbf{Y}^*$ and $n:=\lvert x\rvert>0$;
$\Phi(x)$ pays 1 at the end of step $n$ if and only if $y_1\dots y_n=x$,
and its price is first announced at step $(n-K+1)\vee1$.
See \cite{Vovk:2023Bayesian} for details.

\section{Emergence of objective probability}
\label{sec:objective}

A possible source of probabilities $P_n$ in Protocol~\ref{prot:basic} is a person,
in which case the probabilities are, of course, subjective.
However, another possibility is where these probabilities
are shared by a number of individuals,
in which case they are often referred to as \emph{intersubjective}
\cite{Dawid:1982intersubjective,Gillies:1991}.

A particularly interesting kind of intersubjective probabilities
are those that are output by strategies for Forecaster
that are spelled out and described in books;
see below for examples.
Let us call them \emph{impersonal probabilities}.
They reside not only in individual human minds but also in, e.g., libraries
(becoming part of what Popper referred to as World 3 \cite{Popper:1979}).

Impersonal probabilities approach objective probabilities
when they are described clearly
(in particular, we should understand clearly when they are applicable)
and when they are well-tested
(there have been serious falsification attempts using various strategies for Sceptic,
which did not succeed).
Well-tested strategies for Forecaster may be referred to as probabilistic theories.
They are empirical theories that can be used for prediction.

\begin{remark}
  I am not aware of any examples where subjective or even intersubjective probabilities
  could be called objective without stretching the meaning of ``objective'' too much.
  They are too unstable for this;
  the quality of individuals' or groups' predictions
  could drop suddenly for all kinds of reasons.
\end{remark}

Let me give three examples of such well-tested impersonal probabilities;
many more can be found in, e.g.,
\cite[Sect.~8.4]{Shafer/Vovk:2001} and \cite[Chap.~10]{Shafer/Vovk:2019}.

\subsection{Classical probability}
\label{subsec:classical}

This probabilistic theory is rarely mentioned in books,
because of its unscientific character,
but the classical definition of probability
is a good example of a probabilistic theory \cite[Sect.~2]{Vovk:1993}.

The theory covers a range of devices
(many of which invented for gambling)
that produce one of $m\in M$ outcomes with equal probabilities.
For example, $m=2$ for fair coins, $m=6$ for fair dice, and $m=37$ for fair roulette wheels.
There is no need to distinguish between devices with the same $m$,
and we let $M$ stands for the set of available $m$.
Let us, e.g., set $M:=\{2,6,37\}$ ignoring the devices different from coins, dice, and roulette wheels.
Let us number the outcomes of a device with $m$ possible outcomes starting from 0,
so that the outcomes are $\{0,\dots,m-1\}$
(this convention is standard for coins and roulette wheels but not for dice).

\begin{protocol}\label{prot:classical}\ \\
  \indentI $\K_0:=1$\\
  \indentI FOR $n=1,2,\dots$:\\
    \indentII Forecaster announces $m_n\in M$\\
    \indentII Sceptic announces $f_n:\{0,\dots,m_n-1\}\to[0,\infty]$\\
        \indentIV such that $\frac{1}{m_n}\sum_{y=0}^{m_n-1}f_n(y)=1$\\
    \indentII Reality announces the actual observation $y_n\in\{0,\dots,m_n-1\}$\\
    \indentII $\K_n:=\K_{n-1} f_n(y_n)$.
\end{protocol}

\noindent
Protocol~\ref{prot:classical} describes the process of testing
the classical theory of probability as empirical theory.
The theory depends on a clear understanding of what counts
as a fair coin, die, or roulette wheel,
but modulo such an understanding
it is an empirical statement about reality that can be tested.

\subsection{Genetics}

A more scientific example is the theory of inheriting Daltonism
\cite[Sect.~10.3]{Shafer/Vovk:2019}.
(Daltonism is a form of red-green colour-blindness,
but this mechanism is applicable to numerous other conditions.)

\begin{protocol}\label{prot:Dalton}\ \\
  \indentI $\K_0:=1$\\
  \indentI FOR $n=1,2,\dots$:\\
    \indentII Reality announces $\text{F}_n\in\{\text{N},\text{A}\}$,
        $\text{M}_n\in\{\text{N},\text{C},\text{A}\}$,
        and $\text{S}_n\in\{\text{B},\text{G}\}$\\
    \indentII Forecaster announces $p_n:=p_{\text{F}_n\text{M}_n\text{S}_n}\in[0,1]$\\
    \indentII Sceptic announces $M_n\in\R$\\
    \indentII Reality announces $y_n\in\{0,1\}$\\
    \indentII $\K_n:=\K_{n-1} + M_n(y_n-p_n)$.
\end{protocol}

\noindent
Protocol~\ref{prot:Dalton} describes recording information about a sequence of births.
For each birth, we record the status $\text{F}_n$ of the Father towards Daltonism (Normal or Affected),
the status $\text{M}_n$ of the Mother (Normal, Carrier, or Affected),
the Sex $\text{S}_n$ of the child (Boy or Girl),
and whether the child is affected (1 if ``yes'' and 0 if ``not'').
The genetic theory of inheriting Daltonism gives
the following strategy for Forecaster:
\begin{align*}
  p_{\text{NNB}} = p_{\text{NNG}} = p_{\text{NCG}} = p_{\text{NAG}} = p_{\text{ANB}} = p_{\text{ANG}} &= 0\\
  p_{\text{NCB}} = p_{\text{ACB}} = p_{\text{ACG}} &= 1/2\\
  p_{\text{NAB}} = p_{\text{AAB}} = p_{\text{AAG}} &= 1.
\end{align*}
Protocol~\ref{prot:Dalton} allows us to test the theory and,
after the theory has been well-tested, to use it for prediction.

\subsection{Quantum computing}

Another particularly simple probabilistic theory
is the theory of quantum computing as presented in \cite[Sect.~10.6]{Shafer/Vovk:2019}.

\section{Jeffreys's law}
\label{sec:Jeffreys}

Suppose we have two forecasters,
Forecaster I and Forecaster II,
who announce their predictions
for Reality's move $y_n$ at the same time.
Is it possible for both of them to perform well
while giving very different predictions?
Such a possibility would be worrying
for the picture of emerging objective probabilities
that we discussed in the previous section.
The following result,
which is a version of \emph{Jeffreys's law}
\cite[Sect.~5.2]{Dawid:1984},
shows that this is not possible;
for simplicity, we assume $\lvert\mathbf{Y}\rvert<\infty$
and assume that the two Forecasters output positive probabilities.

\begin{theorem}\label{thm:geometric}
  Suppose Sceptic gambles separately against two Forecasters
  outputting predictions $P\I_n$ and $P\II_n$.
  There is a gambling strategy leading to the capitals
  $\K\I_n$ and $\K\II_n$ satisfying
  \begin{equation}\label{eq:geometric}
    \sqrt{\K\I_n \K\II_n}
    =
    \prod_{i=1}^n
    \frac{1}{H(P\I_i,P\II_i)},
  \end{equation}
  where the \emph{Hellinger integral} $H$ is defined by
  \begin{equation}\label{eq:Hellinger}
    H(P,Q)
    :=
    \sum_{y\in\mathbf{Y}}
    \sqrt{P(\{y\}) Q(\{y\})}
    \in
    [0,1].
  \end{equation}
\end{theorem}

\noindent
The Hellinger integral \eqref{eq:Hellinger}
(also know as the Bhattacharyya coefficient)
is a measure of closeness between $P$ and $Q$,
and it is equal to 1 if and only if $P=Q$.
According to \eqref{eq:geometric},
at least one of the two Forecasters will be discredited
if they do not issue almost identical forecasts in the long run.

Formally, Theorem~\ref{thm:geometric} is a statement about the following protocol.
\begin{protocol}\label{prot:geometric}\ \\
  \indentI $\K_0 := 1$\\
  \indentI FOR $n=1,2,\dots$:\\
    \indentII Forecaster I announces $P\I_n\in\mathfrak{P}(\mathbf{Y})$\\
    \indentII Forecaster II announces $P\II_n\in\mathfrak{P}(\mathbf{Y})$\\
    \indentII Sceptic announces measurable
        $f\I_n:\mathbf{Y}\to[0,\infty]$ and $f\II_n:\mathbf{Y}\to[0,\infty]$\\
        \indentIV such that $P\I_n(f\I_n) = 1$ and $P\II_n(f\II_n) = 1$\\
    \indentII Reality announces the observation $y_n\in\mathbf{Y}$\\
    \indentII $\K\I_n := \K\I_{n-1} f_n (y_n)$\\
    \indentII $\K\II_n := \K\II_{n-1} f_n (y_n)$.
\end{protocol}

\begin{proof}[Proof of Theorem~\ref{thm:geometric}]
  Let us write $H_i$ for $H(P\I_i,P\II_i)$
  and let the values of Forecaster's predictions on the true observations be
  $p\I_i:=P\I_i(\{y_i\})$ and $p\II_i:=P\II_i(\{y_i\})$.
  Informally, Sceptic's strategy is to use the normalized geometric mean $\sqrt{P\I_i P\II_i}/H_i$
  as his alternative prediction.
  When playing against $P\I_i$, his bet is $f\I:=\sqrt{P\I_i P\II_i}/(H_i P\I_i)$,
  and when playing against $P\II_i$, his bet is $f\II:=\sqrt{P\I_i P\II_i}/(H_i P\II_i)$.
  The relative increase in the geometric mean of his capitals at step $i$ is
  \begin{equation*}
    \sqrt{\frac{\sqrt{p\I_i p\II_i}/H_i}{p\I_i}\frac{\sqrt{p\I_i p\II_i}/H_i}{p\II_i}}
    =
    \frac{1}{H_i}.
    \qedhere
  \end{equation*}
\end{proof}

It is interesting that, according to \eqref{eq:geometric},
the geometric mean $\sqrt{\K\I_n \K\II_n}$ can only increase.

As mentioned earlier,
the Hellinger integral $H(P\I_i,P\II_i)$ is a measure of similarity between $P\I_i$ and $P\II_i$.
The closely related \emph{Hellinger distance} $\rho_H(P,Q)$
between probability measures $P$ and $Q$ on $\mathbf{Y}$
is defined by
\[
  \rho_H(P,Q)
  =
  \sum_{y\in\mathbf{Y}}
  \left(
    \sqrt{P(\{y\})} - \sqrt{Q(\{y\})}
  \right)^2
  =
  2-2H(P,Q)
  \in
  [0,2].
\]
Therefore, \eqref{eq:geometric} combined with the inequality $\ln(1+x)\le x$ implies
\begin{equation}\label{eq:geometric-crude}
  \ln\K\I_n+\ln\K\II_n
  \ge
  \sum_{i=1}^n
  \rho_H(P\I_i,P\II_i).
\end{equation}
This, in turn, implies
\begin{equation*}
  \K\I_n+\K\II_n
  \ge
  2\sqrt{\K\I_n\K\II_n}
  \ge
  2
  \exp
  \left(
    \frac12
    \sum_{i=1}^n
    \rho_H(P\I_i,P\II_i)
  \right).
\end{equation*}
On the left-hand side of the last inequality we have Sceptic's total capital
when gambling against both Forecasters (starting from 2).
The total capital can go up or down, unlike $\sqrt{\K\I_n\K\II_n}$.

To obtain an inverse to Theorem~\ref{thm:geometric},
we need to show that if two Forecaster's output similar predictions,
we can gamble successfully against one of them
if and only if we can gamble successfully against the other.
This is vacuous if the two Forecasters output identical predictions,
and an ideal measure of similarity should be close
to the right-hand side of \eqref{eq:geometric}.
To state a simple non-asymptotic result
(Theorem~\ref{thm:optimality} below)
in this direction, we will need the following protocol.

\begin{protocol}\label{prot:optimality}\ \\
  \indentI $\K_0 := 1$\\
  \indentI FOR $n=1,2,\dots$:\\
    \indentII Forecaster I announces $P\I_n\in\mathfrak{P}(\mathbf{Y})$\\
    \indentII Forecaster II announces $P\II_n\in\mathfrak{P}(\mathbf{Y})$\\
    \indentII Sceptic I announces measurable $f\I_n:\mathbf{Y}\to[0,\infty]$
        such that $P\I_n(f\I_n) = 1$\\
    \indentII Sceptic II announces measurable $f\II_n:\mathbf{Y}\to[0,\infty]$
        such that $P\II_n(f\II_n) = 1$\\
    \indentII Reality announces the observation $y_n\in\mathbf{Y}$\\
    \indentII $\K\I_n := \K\I_{n-1} f\I_n (y_n)$\\
    \indentII $\K\II_n := \K\II_{n-1} f\II_n (y_n)$.
\end{protocol}

\noindent
(Protocol~\ref{prot:geometric} can be considered
as the special case of Protocol~\ref{prot:optimality}
in which the two Sceptics are playing as a team.)

\begin{theorem}\label{thm:optimality}
  There is a strategy for Sceptic II in Protocol~\ref{prot:optimality}
  that guarantees
  \begin{equation}\label{eq:optimality}
    \K\II_n
    \ge
    \sqrt{
      \K\I_n
      \prod_{i=1}^n
      \frac{1}{\chi(P\I_i,P\II_i)}
    },
  \end{equation}
  where the \emph{$\chi^2$ integral} $\chi(P,Q)$ is defined by
  \begin{equation}\label{eq:chi}
    \chi(P,Q)
    :=
    \sum_{y\in\mathbf{Y}}
    \frac{Q(\{y\})^2}{P(\{y\})}
    \in
    [1,\infty].
  \end{equation}
\end{theorem}

The inclusion ``${}\in[1,\infty]$'' in \eqref{eq:chi} follows
from the obvious equality
\begin{equation}\label{eq:different-chi}
  \chi(P,Q)
  =
  1 + \rho_{\chi}(P,Q),
\end{equation}
where $\rho_{\chi}$ is the \emph{$\chi^2$ distance}
\[
  \rho_{\chi}(P,Q)
  :=
  \sum_{y\in\mathbf{Y}}
  \frac{(P(\{y\})-Q(\{y\}))^2}{P(\{y\})}.
\]

\begin{proof}[Proof of Theorem~\ref{thm:optimality}]
  Sceptic II's strategy is to invest half of his initial capital
  in the strategy
  \[
    f\II_n(y)
    :=
    \frac{P\II_n(\{y\})}{P\I_n(\{y\})}
    \frac{1}{\chi(P\I_n,P\II_n)}
  \]
  and the other half in the strategy
  \[
    f\II_n(y)
    :=
    \frac{P\I_n(\{y\})}{P\II_n(\{y\})}
    f\I_n(y).
  \]
  Then, setting
  \[
    X
    :=
    \prod_{i=1}^n
    \frac{P\II_i(\{y_i\})}{P\I_i(\{y_i\})},
  \]
  we obtain
  \[
    \K\II_n
    =
    \frac12
    X
    \prod_{i=1}^n
    \frac{1}{\chi(P\I_i,P\II_i)}
    +
    \frac12
    \frac{1}{X}
    \K\I_n.
  \]
  Minimizing the right-hand side over $X$ gives \eqref{eq:optimality}.
\end{proof}

Combining \eqref{eq:optimality}, \eqref{eq:different-chi}, and the inequality $\ln(1+x)\le x$,
we obtain
\begin{equation}\label{eq:optimality-crude}
  \ln\K\II_n
  \ge
  \frac12
  \ln\K\I_n
  -
  \frac12
  \sum_{i=1}^n
  \rho_{\chi}(P\I_n,P\II_n).
\end{equation}
The guarantees \eqref{eq:geometric-crude} and \eqref{eq:optimality-crude}
involve different distances, $\rho_H$ vs $\rho_{\chi}$.
For $P$ and $Q$ that are close to each other,
the latter is approximately 4 times greater:
\begin{align*}
  \rho_{\chi}(P,Q)
  &=
  \sum_{y}
  \frac{(P(\{y\})-Q(\{y\}))^2}{P(\{y\})}\\
  &=
  \sum_{y}
  \left(
    \sqrt{P(\{y\})} - \sqrt{Q(\{y\})}
  \right)^2
  \left(
    1 + \sqrt{\frac{Q(\{y\})}{P(\{y\})}}
  \right)^2
\end{align*}
(and the last equality also shows that $\rho_H\le\rho_{\chi}$).

\bigskip

In scientific applications,
we might have a distinguishing experiment
between two competing probabilistic theories (impersonal Forecasters),
i.e., an experimental setting where the two theories give
different probability forecasts.
Repeating the experiment will allow us to boost any difference between the forecasts;
by Theorem~\ref{thm:geometric}, there is a strategy for Sceptic
leading to an exponential (in the number of repetitions) growth of his total capital
and, therefore, rejecting at least one of the theories.

For further information about Jeffreys's law,
see \cite[Sect.~10.7]{Shafer/Vovk:2019}.
In particular, the optimality, in some sense,
of the Hellinger distance in the context of Jeffreys's law
is shown in \cite[Proposition 10.18]{Shafer/Vovk:2019}.
For further information about Hellinger integral and distance, see \cite{Shiryaev:2016};
in my definitions and notation I mainly follow \cite[Sect.~31]{Borovkov:1998}.
For probability measures that are close to each other,
the standard Kullback--Leibler distance $\rho_K$
is midway between the Hellinger and $\chi^2$ distances
\cite[Sect.~31]{Borovkov:1998},
and $\sum_{i=1}^n\rho_K(P\I_i,P\II_i)$ describes, under natural conditions,
the asymptotic growth rate of Sceptic II's capital when Forecaster I is reliable.

The proof of Theorem~\ref{thm:geometric} follows \cite[proof of Proposition 10.17]{Shafer/Vovk:2019}
(which in turn follows similar proofs in \cite{Vovk:1987} and \cite{Vovk:2009}).
Theorem~\ref{thm:optimality} is generalized in \cite[Theorem 2]{Vovk:2009}.

\section{Frequentist probability}
\label{sec:frequentist}

In this section we will discuss the frequentist interpretation of probability.
It can be regarded as a corollary of the martingale interpretation.
While classical probability is a specific probabilistic theory
(see Sect.~\ref{subsec:classical}),
frequentist probability is a feature of any non-trivial probabilistic theory.
Numerous examples are given in \cite{Shafer/Vovk:2019},
and in this section we only discuss a couple of them.
They are applications of laws of large numbers,
of which we give two standard versions,
infinitary (Sect.~\ref{subsec:infinitary}) and finitary (Sect.~\ref{subsec:finitary}),
and a newer decision-theoretic version (Sect.~\ref{subsec:optimal}).

\subsection{Infinitary result}
\label{subsec:infinitary}

Let us start from stating an infinitary frequentist result
about forecasting bounded observations.

\begin{protocol}\label{prot:SLLN}\ \\
  \indentI $\K_0 := 1$\\
  \indentI FOR $n=1,2,\dots$:\\
    \indentII Forecaster announces $a_n\in\R$\\
    \indentII Sceptic announces $A_n\in\R$\\
    \indentII Reality announces the observation $y_n\in[0,1]$\\
    \indentII $\K_n := \K_{n-1} + A_n (y_n-a_n)$.
\end{protocol}

\begin{proposition}\label{prop:SLLN}
  There is a strategy for Sceptic that guarantees $\K_n\ge0$ for all $n$ and
  either $\K_n\to\infty$ or
  \begin{equation}\label{eq:SLLN}
    \frac1n \sum_{i=1}^n (y_i-a_i) \to 0
    \qquad
    (n\to\infty).
  \end{equation}
\end{proposition}

\noindent
For a simple proof, see \cite[Proposition~1.2]{Shafer/Vovk:2019}.

Let us say that Sceptic \emph{can force} some condition
if he has a strategy that guarantees that $\K_n\ge0$ for all $n$ and
that either $\K_n\to\infty$ or the condition holds.
(In other words, the condition holds unless Forecaster is discredited.)
Then Proposition~\ref{prop:SLLN}
can be restated by saying that Sceptic can force \eqref{eq:SLLN}.

Protocol~\ref{prot:SLLN} might appear specialized,
but in fact we can extract from it
various frequentist corollaries for many other protocols,
including the examples in Sect.~\ref{sec:objective}.

As one such frequentist corollary,
in Protocol~\ref{prot:classical} Sceptic can force
\begin{equation}\label{eq:SLLN-classical}
  \frac1n
  \sum_{i=1}^n
  \left(
    y_i
    -
    \frac{m_i-1}{2}
  \right)
  \to
  0
  \qquad
  (n\to\infty).
\end{equation}
Indeed, any strategy for Sceptic that forces \eqref{eq:SLLN} in Protocol~\ref{prot:SLLN}
can be transformed into a strategy for Sceptic that forces \eqref{eq:SLLN-classical} in Protocol~\ref{prot:classical}:
Sceptic's move $A_n$ in Protocol~\ref{prot:SLLN} fed with
\begin{equation*}
  a_n
  :=
  \frac{m_n-1}{2}
\end{equation*}
should be transformed into the move
\begin{equation}\label{eq:transformation}
  f_n(y)
  :=
  1
  +
  A_n
  \left(
    y - \frac{m_n-1}{2}
  \right)
  / \K_{n-1}
\end{equation}
in Protocol~\ref{prot:classical}.
(Notice that there is a difference between the representation of the evolution of Sceptic's capital
in Protocols~\ref{prot:classical} and~\ref{prot:SLLN}:
in the former it is multiplicative whereas in the latter it is additive,
but the division by $\K_{n-1}$ in \eqref{eq:transformation}
transforms the latter into the former.)

As another frequentist corollary,
suppose Forecaster in Protocol~\ref{prot:Dalton}
always outputs the same $p_n=p$
(or that we record only the steps at which he outputs a given $p$,
such as $p=1/2$).
Then Proposition~\ref{prop:SLLN} implies that Sceptic can force
\[
  \lim_{n\to\infty}
  \frac1n
  \sum_{i=1}^n
  y_i
  =
  p
\]
in that protocol.

\subsection{Finitary result}
\label{subsec:finitary}

A disadvantage of Proposition~\ref{prop:SLLN} is its infinitary character
(cf.\ Sect.~\ref{subsec:other}).
The following protocol is a finite-horizon counterpart of Protocol~\ref{prot:SLLN}
(alternatively, we can just ignore the steps after step $N$ in Protocol~\ref{prot:SLLN}).

\begin{protocol}\label{prot:WLLN}\ \\
  \indentI $\K_0 := 1$\\
  \indentI FOR $n=1,\dots,N$:\\
    \indentII Forecaster announces $a_n\in\R$\\
    \indentII Sceptic announces $A_n\in\R$\\
    \indentII Reality announces the observation $y_n\in[0,1]$\\
    \indentII $\K_n := \K_{n-1} + A_n (y_n-a_n)$.
\end{protocol}

\begin{proposition}[version of Lemma~2.7 in \cite{Shafer/Vovk:2019}]\label{prop:WLLN}
  Let $\delta>0$.
  There is a strategy for Sceptic in Protocol~\ref{prot:WLLN}
  that guarantees $\K_n\ge0$ for all $n$ and
  either $\K_N\ge1/\delta$ or
  \begin{equation}\label{eq:WLLN}
    \frac1N \sum_{i=1}^N (y_i-a_i) < (\delta N)^{-1/2}.
  \end{equation}
\end{proposition}

The interesting case is when the $\delta$ in the statement of Proposition~\ref{prop:WLLN} is small.
The proposition then says that either Forecaster is discredited at level $1/\delta$
or the average forecast is close to the average observation.

\begin{proof}
  The proof is based on the version
  \[
    \K_n
    :=
    1
    +
    \frac1N
    \left(
      \sum_{i=1}^n (y_i - a_i)
    \right)^2
    -
    \frac1N
    \sum_{i=1}^n (y_i - a_i)^2
  \]
  of Kolmogorov's martingale \cite[Sect.~3.1]{Shafer/Vovk:2019},
  which is always nonnegative.
  It ends with
  \[
    \K_N
    \ge
    \frac1N
    \left(\sum_{i=1}^N (y_i - a_i)\right)^2
  \]
  at the end of step $N$.
  It remain to notice that $\K_N<1/\delta$
  implies \eqref{eq:WLLN}.
\end{proof}

A finitary counterpart of the infinitary notion of forcing
is upper game-theoretic probability:
\begin{multline*}
  \UP(E)
  :=
  \inf
  \bigl\{
    \epsilon>0
    \mid
    \text{$\exists$ valid strategy for Sceptic}\\
    \text{such that $E$ happens or $\K_N\ge1/\epsilon$}
  \bigr\},
\end{multline*}
where a strategy for Sceptic is valid if it guarantees $\K_n\ge0$ for all $n$.
The statement of Proposition~\ref{prop:WLLN} can be written as
\[
  \UP
  \left(
    \frac1N \sum_{i=1}^N (y_i-a_i) \ge \epsilon
  \right)
  \le
  \frac{1}{\epsilon^2 N}.
\]
Upper game-theoretic probability is closely (but non-trivially)
related to measure-theoretic probability;
cf.\ \cite[Chap.~9]{Shafer/Vovk:2019}.

Similarly to Proposition~\ref{prop:SLLN},
Proposition~\ref{prop:WLLN} has numerous frequentist implications
for specific probabilistic theories.

\subsection{Nearly optimal decision making}
\label{subsec:optimal}

The way of testing two-steps-ahead predictions given by Protocol~\ref{prot:2-steps-point-test}
carries over in a natural way to Protocol~\ref{prot:DM}
(the idea was mentioned at the end of Sect.~\ref{subsec:2-steps};
details omitted but can be found in \cite[Protocol~7.4]{Vovk:2023Bayesian}
for the finite-horizon case).
Let
\begin{equation}\label{eq:Loss}
  \Loss_N
  :=
  \sum_{n=1}^N
  \lambda_n(d_n,y_n\dots y_{n+K-1})
  \in
  [0,N]
\end{equation}
be the loss suffered by Decision Maker in Protocol~\ref{prot:DM},
and let $\Loss^{\text{B}}_N$ be the loss \eqref{eq:Loss}
where $d_n$ are the Bayes optimal decisions \eqref{eq:B}
rather than Decision Maker's actual moves.
We have the following law of large numbers for decision making
in the testing protocol extending Protocol~\ref{prot:DM};
it assumes $\lvert\mathbf{Y}\rvert<\infty$.

\begin{theorem}\label{thm:optimal}
  Let $\epsilon>0$ and $N$ (time horizon) be a natural number.
  The Bayes optimal strategy for Decision Maker in Protocol~\ref{prot:DM}
  satisfies
  \begin{equation*}
    \UP
    \left(
      \frac1N
      \left(
        \Loss_N^{\text{B}} - \Loss_N
      \right)
      \ge
      \epsilon
    \right)
    \le
    \exp
    \left(
      -\frac{N\epsilon^2}{8K^2}
    \right).
  \end{equation*}
\end{theorem}

\noindent
For a proof, see \cite[Appendix~C]{Vovk:2023Bayesian}.

Theorem~\ref{thm:optimal} implies that, for $K\ll\sqrt{N}\epsilon$,
the loss of the Bayes optimal strategy will be less than,
or approximately equal to,
the loss of any other strategy for Decision Maker.

\section{Conclusion}

This paper describes a non-standard interpretation of probability,
the martingale interpretation.
Two other popular interpretations are embedded into it:
classical, discussed in Sect.~\ref{subsec:classical}, and frequentist, discussed in Sect.~\ref{sec:frequentist}.
Unlike the psychological (Remark~\ref{rem:psychology}) and logical (Remark~\ref{rem:logical}) interpretations,
it connects probabilities asserted for events with the outcomes of those events.

\subsection*{Acknowledgements}

Many thanks to Barbara Vantaggi and the other organizers of MSP 2023
for inviting me to give a talk.

\end{document}